\documentclass{article}
\usepackage{iclr2027_conference,times}
\usepackage{amsmath,amssymb,amsthm,mathtools,bm}
\usepackage{booktabs,multirow,array}
\usepackage{graphicx}
\usepackage{algorithm}
\usepackage{algpseudocode}
\usepackage{xcolor}
\usepackage{microtype}
\usepackage{enumitem}
\usepackage{url}
\usepackage{hyperref}
\usepackage[nameinlink,capitalise,noabbrev]{cleveref}
\hypersetup{colorlinks=true,citecolor=blue!55!black,linkcolor=blue!55!black,urlcolor=blue!55!black}
\allowdisplaybreaks

\newtheorem{theorem}{Theorem}[section]
\newtheorem{proposition}[theorem]{Proposition}
\newtheorem{lemma}[theorem]{Lemma}
\newtheorem{corollary}[theorem]{Corollary}

\theoremstyle{definition}

\newcommand{\E}{\mathbb{E}}
\newcommand{\Pp}{\mathbb{P}}
\newcommand{\R}{\mathbb{R}}
\newcommand{\F}{\mathrm{F}}
\newcommand{\op}{\mathrm{op}}

\newcommand{\rank}{\operatorname{rank}}

\newcommand{\argmin}{\operatorname*{arg\,min}}
\newcommand{\argmax}{\operatorname*{arg\,max}}
\newcommand{\cI}{\mathcal{I}}
\newcommand{\cG}{\mathcal{G}}
\newcommand{\cT}{\mathcal{T}}

\newcommand{\DeltaCU}{\Delta_{\mathrm{CU}}}
\newcommand{\widehatDeltaCU}{\widehat{\Delta}_{\mathrm{CU}}}
\newcommand{\ECU}{E_{\mathrm{CU}}}
\newcommand{\inner}[2]{\left\langle #1,#2\right\rangle}
\newcommand{\norm}[2][]{\left\lVert #2\right\rVert_{#1}}
\newcommand{\abs}[1]{\left|#1\right|}

\newcommand{\LRD}{\textsc{LR-D2}}
\newcommand{\NOTD}{\textsc{Seeded-LR-D2}}
\newcommand{\RefineD}{\textsc{Refine-LR-D2}}
\newcommand{\epsr}{\varepsilon_r}

\title{High-Dimensional Nonparametric Change-Point Detection\\via Low-Rank Degree-Two Density Projection}
\author{Guoqing Zhang\\
Operations Research Program\\
North Carolina State University\\
Raleigh, NC, USA\\
{\normalfont\texttt{Gzhang25@ncsu.edu}}
\And
Zhaixin Chen\\
H. Milton Stewart School of ISyE\\
Georgia Institute of Technology\\
Atlanta, GA, USA\\
{\normalfont\texttt{zxchen08@gatech.edu}}}
\hypersetup{pdftitle={High-Dimensional Nonparametric Change-Point Detection via Low-Rank Degree-Two Density Projection},pdfauthor={Guoqing Zhang and Zhaixin Chen}}
\iclrfinalcopy
\begin{document}
\maketitle
\fancyhead[L]{Author draft for ICLR 2027}
\begin{abstract}
Detecting distributional changes in high dimension is difficult when neither the pre-change nor post-change density is parametrically specified.  We introduce a representation-based approach that retains all degree-at-most-two density information while replacing density estimation by matrix mean estimation.  For observations in $[-1,1]^d$, a symmetric feature matrix $H_2(X)\in\R^{(d+1)\times(d+1)}$ is constructed so that $M(f)=\E_f H_2(X)$ is an isometric encoding of the degree-two orthogonal projection of the density.  We scan matrix CUSUMs after rank-$r$ truncation, exploiting the low rank of the projected jump rather than sparsity of individual coordinates.  The resulting \LRD{} estimator has a tent-shaped population objective and a nonasymptotic operator-norm analysis whose leading stochastic term scales as $\sqrt{rd\log(nd)}$.  For multiple changes, we give a seeded narrowest-over-threshold procedure and prove exact recovery by an induction that preserves an isolating interval for every undetected change.  A cross-fitted scalar refinement learns the changing low-rank direction on one fold and localizes on the other, attaining $\widetilde O_{\Pp}(\kappa^{-2})$ error; a matching Le Cam lower bound shows optimality up to logarithms.  A geometrically $\beta$-mixing extension follows from a dependent matrix Bernstein inequality.  Experiments with ambient dimension up to $200$, a three-change $d=100$ sequence, and a $128$-feature human-activity benchmark show that the method remains computationally practical and accurately detects pure dependence changes that are invisible to mean CUSUMs.
\end{abstract}

\section{Introduction}

A change point is a time at which the law generating a sequence changes.  In modern applications, the observation at each time can be a high-dimensional vector and the change need not be a mean shift: dependence, marginal shape, or interactions may change while every coordinate mean remains fixed.  Classical multivariate procedures then face two simultaneous obstacles.  First, direct nonparametric density estimation is subject to the curse of dimensionality.  Second, an unrestricted second-order change contains $\Theta(d^2)$ coordinates, so a Frobenius scan can aggregate overwhelming noise even when the scientifically meaningful change is low-dimensional.

We address these obstacles through a fixed representation of the density.  Relative to the uniform measure on $[-1,1]^d$, the degree-two projection contains coordinate means, marginal quadratic components, and pairwise interactions.  With an appropriate normalization, all these coefficients form a symmetric $(d+1)\times(d+1)$ matrix $M(f)$.  Crucially, there is a single-observation feature map $H_2$ satisfying
\[
M(f)=\E_f H_2(X),\qquad
\norm[\F]{M(f)-M(g)}=\norm[L^2]{P_2(f-g)}.
\]
Thus a nonparametric change visible to the degree-two projection becomes a matrix-mean change, without estimating either density.  When the projected jump matrix is low rank, rank truncation denoises the matrix CUSUM while preserving its population signal.

This viewpoint differs from coordinate sparsity.  A dense correlation pattern generated by a few latent interaction directions may have many nonzero entries but small rank.  It also provides a concrete representation-learning interpretation: $H_2$ is a deterministic second-order feature map, and the data select the changing subspace through a local eigendecomposition.  The matrix dimension grows only linearly with $d$, and the retained rank controls the effective stochastic complexity.

\paragraph{Contributions.}
Our main contributions are fourfold.
\begin{enumerate}[leftmargin=1.35em,itemsep=1pt,topsep=2pt]
\item We derive an exact isometric matrix encoding of the degree-two density projection and show that, on any interval containing one change, every population matrix CUSUM is a positive scalar multiple of the same jump matrix.  The resulting squared population score has an exact linear localization margin.
\item We prove a deterministic low-rank perturbation inequality and combine it with matrix Bernstein concentration.  Under a bounded-density condition, the leading normalized CUSUM noise is $O_{\Pp}(\sqrt{d\log(nd)})$ in operator norm, yielding the signal requirement $\Delta\kappa^2\gtrsim rd\log(nd)$ up to constants and lower-order terms.
\item We propose \NOTD{}, a multiscale interval-deletion algorithm.  Its multiple-change theorem is proved by induction: no-change intervals are inactive, every remaining change has an active balanced isolating interval, and the shortest active interval contains exactly one change.  We then introduce \RefineD{}, a cross-fitted direction-and-localization step with near-minimax $\kappa^{-2}$ error.
\item We provide reproducible synthetic and semi-synthetic experiments.  The estimator is run at $d\in\{20,50,100,200\}$, detects three changes in a $d=100$ sequence, and processes a $128$-feature UCI human-activity task.  These settings directly test the claimed high-dimensional operating regime.
\end{enumerate}

\paragraph{Related work.}
High-dimensional mean-change methods exploit sparse projections or sparsified segmentation \citep{wang2018sparse,cho2015sparsified,enikeeva2019high}.  Nonparametric multivariate and functional change-point methods include kernel, graph, and local-neighborhood constructions \citep{harchaoui2007retrospective,gretton2012kernel,madrid2022optimal,madrid2022functional,aue2018detecting}.  Our method is complementary: it is nonparametric at the density level but targets changes identifiable through a prescribed polynomial projection, with low matrix rank as the structural regularizer.  For multiple changes we use ideas from wild, narrowest-over-threshold, and seeded binary segmentation \citep{fryzlewicz2014wild,baranowski2019narrowest,kovacs2023seeded}.  The two-stage localization architecture and induction-based consistency argument are also related to recent functional-regression change-point analysis \citep{kumar2024functional}.  Matrix concentration is based on \citet{tropp2012user}; the dependent extension uses \citet{banna2016bernstein}.

\section{Degree-two density representation}
\label{sec:representation}

\subsection{An isometric feature matrix}
Let $\mu=\operatorname{Unif}([-1,1]^d)$ and let $f=dP/d\mu$.  Define the first Legendre polynomials
\[
p_0(x)=1,\qquad p_1(x)=x,\qquad p_2(x)=\frac{3x^2-1}{2},
\]
whose squared $L^2(\mu)$ norms are $1$, $1/3$, and $1/5$.  The degree-at-most-two projection is
\begin{equation}
P_2f(x)=1+\sum_{j=1}^d a_jx_j+\sum_{j=1}^d b_jp_2(x_j)
+\sum_{1\le i<j\le d}\gamma_{ij}x_ix_j,
\label{eq:p2-main}
\end{equation}
where $a_j=3\E_fX_j$, $b_j=5\E_fp_2(X_j)$, and $\gamma_{ij}=9\E_f(X_iX_j)$.

For $x\in[-1,1]^d$, define the symmetric matrix $H_2(x)\in\R^{(d+1)\times(d+1)}$ by
\begin{align}
[H_2(x)]_{00}&=1,\nonumber\\[-1.0em]
[H_2(x)]_{0j}=[H_2(x)]_{j0}&=\frac{3x_j}{\sqrt6},\qquad
[H_2(x)]_{jj}=\sqrt5\,p_2(x_j),\nonumber\\[-0.6em]
[H_2(x)]_{ij}=[H_2(x)]_{ji}&=\frac{3x_ix_j}{\sqrt2},\qquad i<j.
\label{eq:h2-main}
\end{align}
Set $M(f)=\E_fH_2(X)$.  The normalization accounts for the double occurrence of off-diagonal entries in the Frobenius norm.

\begin{proposition}[Isometric representation]
\label{prop:isometry-main}
For any densities $f,g$ with respect to $\mu$,
\[
\norm[\F]{M(f)-M(g)}=\norm[L^2(\mu)]{P_2(f-g)}.
\]
Moreover, $\norm[\op]{H_2(x)}\le3d$ for all $x$, and
$\E_{\mu}[H_2(X)^2]=(1+d/2)I_{d+1}$.
\end{proposition}
The proof is a direct orthogonality calculation and is given in \cref{app:representation}.  The first identity is the core reduction: estimating $P_2f$ is equivalent to estimating a matrix mean.  It also makes the limitation explicit.  If $P_2(f-g)=0$, no method using only degree-two features can distinguish $f$ and $g$ at the population level.

\subsection{Piecewise distributions and projected jumps}
Let $X_1,\ldots,X_n\in[-1,1]^d$ be independent for the main theory.  There are change points
\[
0=\eta_0<\eta_1<\cdots<\eta_K<\eta_{K+1}=n
\]
and segment densities $f_0,\ldots,f_K$ such that $X_i\sim f_k$ for $\eta_k<i\le\eta_{k+1}$.  Write
\[
M_k=M(f_k),\qquad D_k=M_k-M_{k-1},\qquad
\kappa_k=\norm[\F]{D_k}.
\]
We assume $\rank(D_k)\le r$ and define
\[
\kappa=\min_k\kappa_k,\qquad
\Delta=\min_{1\le k\le K+1}(\eta_k-\eta_{k-1}).
\]
The rank assumption concerns the change, not the segment matrices themselves.  For example, changing the dependence of one coordinate pair produces a rank-two symmetric jump even when $d$ is large.

\section{Low-rank matrix CUSUM}
\label{sec:cusum}

On an interval $(s,e]$ and candidate split $s<t<e$, define
\begin{align}
\widehat M_L^{s,e}(t)&=\frac1{t-s}\sum_{i=s+1}^tH_2(X_i),&
\widehat M_R^{s,e}(t)&=\frac1{e-t}\sum_{i=t+1}^eH_2(X_i),\nonumber\\[-0.5em]
\widehatDeltaCU^{s,e}(t)&=
\sqrt{\frac{(t-s)(e-t)}{e-s}}
\left\{\widehat M_R^{s,e}(t)-\widehat M_L^{s,e}(t)\right\}.
\label{eq:matrix-cusum-main}
\end{align}
Let $A_{(r)}$ denote a best rank-$r$ approximation of $A$, obtained by retaining the $r$ eigencomponents with largest absolute eigenvalues for symmetric $A$.  Our scan score is
\begin{equation}
\widehat W_r^{s,e}(t)=\norm[\F]{\bigl(\widehatDeltaCU^{s,e}(t)\bigr)_{(r)}}.
\label{eq:lr-score-main}
\end{equation}

Suppose $(s,e]$ contains one change $b$ from $M_0$ to $M_1$ and write $D=M_1-M_0$, $\kappa=\norm[\F]{D}$.  The population CUSUM has the exact form
\begin{equation}
\DeltaCU^{s,e}(t)=\E\widehatDeltaCU^{s,e}(t)=a_b^{s,e}(t)D,
\label{eq:population-scalar-main}
\end{equation}
where
\begin{equation}
a_b^{s,e}(t)=
\begin{cases}
\dfrac{e-b}{\sqrt{e-s}}\sqrt{\dfrac{t-s}{e-t}},&t<b,\\[1.1ex]
\sqrt{\dfrac{(b-s)(e-b)}{e-s}},&t=b,\\[1.1ex]
\dfrac{b-s}{\sqrt{e-s}}\sqrt{\dfrac{e-t}{t-s}},&t>b.
\end{cases}
\label{eq:a-main}
\end{equation}
Hence all candidate population matrices share the same singular vectors and rank.  The coefficient increases to $b$ and decreases afterward.  More strongly,
\begin{equation}
\norm[\F]{\DeltaCU^{s,e}(b)}^2-\norm[\F]{\DeltaCU^{s,e}(t)}^2
=
\begin{cases}
\dfrac{(e-b)(b-t)}{e-t}\kappa^2,&t<b,\\[1.0ex]
\dfrac{(b-s)(t-b)}{t-s}\kappa^2,&t>b.
\end{cases}
\label{eq:exact-margin-main}
\end{equation}
This exact linear margin drives localization.  \Cref{fig:score-curve} illustrates that the rank-two empirical score follows the population tent in a $d=100$ dependence-change example.

\begin{figure}[t]
\centering
\includegraphics[width=0.96\linewidth]{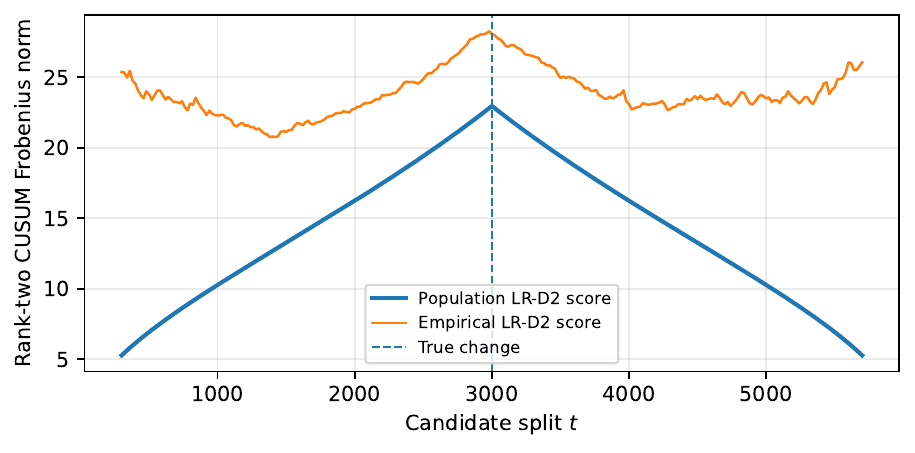}
\caption{Population and empirical rank-two degree-two CUSUM scores for one $d=100$ Gaussian-copula change.  Coordinate means and marginal distributions are unchanged; only the dependence of the first two coordinates changes.}
\label{fig:score-curve}
\end{figure}

The low-rank step is controlled by a deterministic inequality.
\begin{proposition}[Low-rank perturbation]
\label{prop:pert-main}
If $\widehat A=A+E$, then
\[
\norm[\F]{\widehat A_{(r)}-A}
\le \norm[\F]{A-A_{(r)}}+2\sqrt{2r}\norm[\op]{E}.
\]
Consequently, on a one-change interval with $\rank(D)\le r$,
\[
\norm[\F]{\bigl(\widehatDeltaCU^{s,e}(t)\bigr)_{(r)}-a_b^{s,e}(t)D}
\le c_r\sqrt r\norm[\op]{\ECU^{s,e}(t)},\qquad c_r=2+\sqrt2,
\]
where $\ECU=\widehatDeltaCU-\DeltaCU$.
\end{proposition}
The approximation-tail term for an approximately rank-$r$ jump is stated in \cref{app:approx-rank}.

\section{Multiple changes and cross-fitted refinement}
\label{sec:method}

\subsection{Seeded interval deletion}
Fix a minimum scale $m$.  For dyadic lengths $h\in\{m,2m,4m,\ldots\}$, form intervals of length $h$ whose starting points are spaced by $h/2$; denote the collection by $\cI_m$.  On $I=(s,e]$, candidates are restricted to its central half,
$\cT_I=\{t:s+h/4\le t\le e-h/4\}$.  Let
\[
\widehat b_I\in\argmax_{t\in\cT_I}\widehat W_r^{s,e}(t),\qquad
\widehat S_I=\widehat W_r^{s,e}(\widehat b_I).
\]
The interval-deletion procedure in \cref{alg:seeded} repeatedly takes the shortest interval above threshold and removes that whole interval from further recursion.  Removing an interval rather than splitting exactly at $\widehat b_I$ makes the induction transparent: once the chosen interval is known to contain one change and to be shorter than $\Delta$, it cannot remove another change.

\begin{algorithm}[t]
\caption{\NOTD{}: seeded low-rank degree-two detection}
\label{alg:seeded}
\begin{algorithmic}[1]
\Require Data $X_{1:n}$, rank $r$, minimum scale $m$, threshold $\tau$
\State Form $\cI_m$ and compute $(\widehat b_I,\widehat S_I)$ for all $I\in\cI_m$
\State Initialize active component list $\mathcal C\gets\{(0,n]\}$ and output $\widehat{\mathcal B}\gets\varnothing$
\While{some $I\in\cI_m$ is contained in a component in $\mathcal C$ and $\widehat S_I>\tau$}
  \State Choose a shortest such $I^*=(s^*,e^*]$; append $(I^*,\widehat b_{I^*})$ to $\widehat{\mathcal B}$
  \State Replace the containing component $(a,c]$ by the nonempty components $(a,s^*]$ and $(e^*,c]$
\EndWhile
\State \Return $\widehat{\mathcal B}$
\end{algorithmic}
\end{algorithm}

\subsection{Cross-fitted local refinement}
A selected interval $I_k=(s_k,e_k]$ identifies one change but need not provide the optimal rate.  We expand it to a one-change window and split the observations by parity.  On the pilot fold, scan the window and normalize the rank-$r$ CUSUM at its maximizer:
\[
\widehat V_k=\frac{(\widehatDeltaCU^{\mathrm{pilot}})_{(r)}}
{\norm[\F]{(\widehatDeltaCU^{\mathrm{pilot}})_{(r)}}}.
\]
On the held-out fold, project $Z_i=\inner{H_2(X_i)}{\widehat V_k}$.  Disjoint outer anchor blocks estimate the left and right projected means, $\widehat\mu_{L,k}$ and $\widehat\mu_{R,k}$.  Over the central search region, define
\begin{equation}
Q_k(t)=\sum_{i\le t}(Z_i-\widehat\mu_{L,k})^2
+\sum_{i>t}(Z_i-\widehat\mu_{R,k})^2,
\qquad
\widetilde\eta_k\in\argmin_tQ_k(t).
\label{eq:refine-objective-main}
\end{equation}
The pilot direction and held-out noise are independent.  Conditional on a well-aligned direction, the problem reduces to a scalar mean change whose effective jump is at least a constant fraction of $\kappa_k$.

\section{Theory}
\label{sec:theory}

Assume throughout this section that every segment density is bounded by $L$ relative to $\mu$.  Let $\cG$ be all triples $(s,t,e)$ scanned by the seeded collection and set
\[
\ell_{\cG,\delta}=\log\frac{2(d+1)|\cG|}{\delta},\qquad
\lambda_m(\delta)=
\sqrt{2L(1+d/2)\ell_{\cG,\delta}}
+\frac{8d\ell_{\cG,\delta}}{\sqrt m}.
\]
For half-overlapping dyadic intervals, $|\cG|=O(n\log(n/m))$.

\begin{theorem}[Uniform CUSUM and low-rank control]
\label{thm:uniform-main}
With probability at least $1-\delta$,
\[
\sup_{(s,t,e)\in\cG}\norm[\op]{\ECU^{s,e}(t)}\le\lambda_m(\delta).
\]
On every scanned interval containing one change with rank at most $r$,
\[
\sup_{t\in\cT_I}
\abs{\widehat W_r^{s,e}(t)-a_b^{s,e}(t)\kappa}
\le \epsr,
\qquad \epsr=c_r\sqrt r\lambda_m(\delta).
\]
\end{theorem}
The first term in $\lambda_m$ is $\asymp\sqrt{Ld\log(nd/\delta)}$ and does not grow with interval length because the CUSUM weights have squared sum one.  The second term is the bounded-summand correction and is lower order when $m\gg d\log(nd)$.

\begin{theorem}[Exact multiple-change recovery by induction]
\label{thm:multiple-main}
Suppose $m\le\Delta/8$, $\rank(D_k)\le r$, and the threshold satisfies
\begin{equation}
\sqrt r\lambda_m(\delta)<\tau<
\frac{\sqrt\Delta}{8\sqrt2}\kappa-\epsr.
\label{eq:threshold-main}
\end{equation}
Then, on the event of \cref{thm:uniform-main}, \NOTD{} returns exactly $K$ intervals.  They are in one-to-one correspondence with the true change points; every selected interval contains exactly one $\eta_k$ and has length at most $\Delta/4$.  In particular, its recorded maximizer satisfies $|\widehat b_k-\eta_k|\le\Delta/4$.
\end{theorem}
The proof, in \cref{app:multiple-proof}, follows the requested induction.  At each stage: (i) an interval with no remaining change has score below $\tau$; (ii) every remaining change has a balanced isolating seeded interval of length in $[\Delta/8,\Delta/4]$ whose score exceeds $\tau$; and (iii) the shortest active interval is therefore shorter than $\Delta$ and must contain exactly one change.  Deleting it leaves every other change at least $3\Delta/4$ from the new component boundary, preserving the induction hypothesis.

\begin{corollary}[Preliminary localization on a one-change window]
\label{cor:prelim-main}
Let $(s,e]$ contain one change $b$, let $N=e-s$, and suppose $b$ and all candidates are at least $\alpha N$ from the endpoints.  Under the uniform event and exact rank,
\[
|\widehat b-b|
\le \frac{2c_r}{\alpha}\frac{\sqrt{Nr}\,\lambda_m(\delta)}{\kappa}.
\]
\end{corollary}
This rate is sufficient to isolate changes but is not generally minimax.  The refinement removes the factor $\sqrt N$.

For the refinement, assume that for every deterministic $V$ with $\norm[\F]{V}=1$, the centered scalar projection $\xi=\inner{H_2(X)-M_k}{V}$ satisfies the Bernstein moment-generating-function condition
\begin{equation}
\log\E e^{\lambda\xi}
\le \frac{\lambda^2\nu^2}{2(1-B|\lambda|)},
\qquad |\lambda|<B^{-1}.
\label{eq:scalar-mgf-main}
\end{equation}
In particular, for independent projections this implies
\begin{equation}
\Pp\!\left(\left|\sum_{i=1}^q\xi_i\right|\ge x\right)
\le2\exp\left[-c\min\left\{\frac{x^2}{q\nu^2},\frac{x}{B}\right\}\right],
\label{eq:scalar-bernstein-main}
\end{equation}
as well as the maximal line-crossing inequality used in the proof.  Bounded observations always give conservative parameters of order $d$; for independent sub-Gaussian coordinates, quadratic-form concentration can make them dimension-free up to constants.

\begin{theorem}[Cross-fitted localization]
\label{thm:refine-main}
Consider a one-change refinement window with pure outer anchor blocks of size $q$.  Suppose the pilot fold yields $\norm[\F]{\widehat V}=1$ and
$\inner{D}{\widehat V}\ge\kappa/2$.  If
\[
q\ge C\left(\frac{\nu^2}{\kappa^2}+\frac{B}{\kappa}\right)
\log\frac{K}{\delta},
\]
then, simultaneously for all $K$ windows, with probability at least $1-\delta$,
\begin{equation}
|\widetilde\eta_k-\eta_k|
\le C\left(\frac{\nu^2}{\kappa_k^2}+\frac{B}{\kappa_k}\right)
\log\frac{K}{\delta}+1.
\label{eq:refined-rate-main}
\end{equation}
Moreover, the direction condition holds whenever the pilot population CUSUM at its maximizer is at least $8\epsr$.
\end{theorem}
The extra $1$ accounts for mapping the parity fold back to the original time index.  In the usual small-jump regime the first term dominates, giving $\widetilde O_{\Pp}(\kappa_k^{-2})$.

\begin{theorem}[Minimax lower bound]
\label{thm:lower-main}
There are constants $c,c_0>0$ such that, for $0<\kappa\le c_0$ and a single change separated from the endpoints by at least $c\kappa^{-2}$,
\[
\inf_{\widehat\eta}\sup_{P\in\mathcal P(\kappa)}
\E_P|\widehat\eta-\eta(P)|\ge\frac{c}{\kappa^2},
\]
where $\mathcal P(\kappa)$ is a class of densities on $[-1,1]^d$ whose degree-two jump has Frobenius norm $\kappa$ and rank two.
\end{theorem}
The construction uses $f_\theta(x)=1+3\theta x_1x_2$, for which the jump matrix has rank two and Frobenius norm $|\theta|$.  Moving the change by $h$ alters the likelihood in only $h$ observations and gives KL divergence $O(h\theta^2)$; Le Cam's method with $h\asymp\theta^{-2}$ proves the result.

\paragraph{Temporal dependence.}
If $Z_i=H_2(X_i)-\E H_2(X_i)$ is uniformly geometrically $\beta$-mixing and $\norm[\op]{Z_i}\le B_d$, the weighted CUSUM is a dependent self-adjoint matrix sum.  The Bernstein inequality of \citet{banna2016bernstein} yields a replacement $\lambda_m^{\mathrm{mix}}$ involving its dependent variance proxy and a logarithmic mixing penalty.  Since \cref{thm:multiple-main} is deterministic conditional on uniform CUSUM control, the entire induction carries over after replacing $\lambda_m$ by $\lambda_m^{\mathrm{mix}}$; see \cref{app:mixing}.

\section{Experiments}
\label{sec:experiments}

All experiments use the same $H_2$ implementation and symmetric eigentruncation.  The preliminary scans use the odd observations; the even observations are reserved for refinement.  Full settings, seeds, and additional tables are in \cref{app:experiments}.

\paragraph{Dimension scaling.}
We generate Gaussian-copula observations with uniform marginals.  Only the dependence between coordinates $1$ and $2$ changes, from copula correlation $0$ to $0.85$ at $n/2$; all coordinate means and marginals remain unchanged.  The projected jump is rank two.  We set $n=60d$ and run $50$ replications for $d\in\{20,50,100,200\}$.  The median absolute errors are reported in \cref{tab:single-main} and \cref{fig:dimension}.  \LRD{} remains near the oracle pair projection, whereas the unregularized full degree-two score degrades with $d$ and mean CUSUM is uninformative.  At $d=200,n=12000$, the refined median error is $4$ and the mean runtime for the complete five-method replicate is $0.92$ seconds on the evaluation machine.

\begin{table}[t]
\centering
\caption{Median absolute localization error over $50$ single-change replications.  ``Full D2'' uses the Frobenius norm without rank truncation; ``Oracle'' knows the affected interaction.}
\label{tab:single-main}
\small
\setlength{\tabcolsep}{4.5pt}
\begin{tabular}{rrrrrr}
\toprule
$d$ & \LRD{} pre. & \LRD{} ref. & Full D2 & Mean & Oracle\\
\midrule
20  & 6 & 11 & 19 & 419  & 4\\
50  & 6 & 10 & 16 & 684  & 4\\
100 & 7 & 8  & 41 & 1416 & 6\\
200 & 4 & 4  & 96 & 3166 & 4\\
\bottomrule
\end{tabular}
\end{table}

\begin{figure}[t]
\centering
\includegraphics[width=0.97\linewidth]{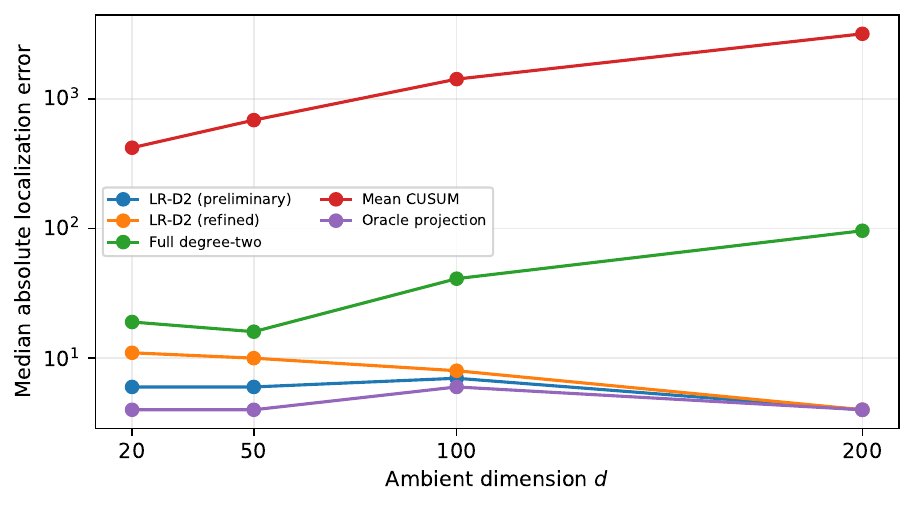}
\caption{Median single-change error versus ambient dimension.  Sample size scales as $n=60d$.  The low-rank estimator remains stable through $d=200$, while the full degree-two and mean scans accumulate high-dimensional noise.}
\label{fig:dimension}
\end{figure}

\paragraph{Multiple changes in $d=100$.}
We use $n=24000$ and changes at $(6000,12000,18000)$, alternating the first-pair copula correlation between $0.95$ and $-0.95$.  A threshold of $30$ is rounded above the maximum of six independent null scans.  Across $12$ replications, \NOTD{} estimates $K=3$ every time.  The median Hausdorff error is $2$ before and $1$ after refinement; all refined estimates are within $10$ time points, and mean runtime is $8.93$ seconds.

\paragraph{UCI human activity, 128 features.}
From the UCI smartphone activity data \citep{anguita2013har}, we take $400$ walking and $400$ walking-downstairs observations and select the $128$ highest-variance features.  An independent Rademacher sign is applied to every row, eliminating first moments while preserving all degree-two products.  Over $50$ randomized replications, refined \LRD{} has median error $0$, exact recovery rate $72\%$, and error at most $4$ in $94\%$ of runs.  Full D2 has median $0$ but does not benefit from rank regularization; mean CUSUM has median error $256$ and never falls within $4$.

\section{Discussion and limitations}

Degree-two projection turns a broad class of nonparametric changes into a structured matrix-mean problem.  The representation is especially useful when a high-dimensional dependence change is generated by a small number of interaction directions.  The multiple-change induction and cross-fitted refinement separate three tasks that are often conflated: detecting a signal, isolating each change, and obtaining the optimal local rate.

The method intentionally does not detect every possible density change.  A change orthogonal to all degree-at-most-two polynomials is invisible, and higher-degree extensions require tensor structure or another finite feature family.  Exact theory assumes low-rank projected jumps; approximate rank introduces a transparent singular-value-tail bias.  The sharp refinement theorem is stated for independent observations, while the current dependent theory directly covers the preliminary matrix scan and multiple-change induction.  Finally, the empirical study is designed to validate scaling and identifiability rather than claim dominance on every change-point benchmark.  Richer learned feature maps and adaptive rank selection are important next steps.

\section*{Ethics statement}
This work develops a statistical method and uses public or synthetic data.  The UCI human-activity data are used only for a methodological benchmark; no attempt is made to infer sensitive personal attributes.  Failure to detect a change should not be interpreted as evidence that two underlying distributions are identical, because the method tests only the chosen degree-two representation.

\section*{Reproducibility statement}
The appendices give complete proofs of all stated results, the exact feature normalization, pseudocode, simulation distributions, tuning rules, and additional numerical summaries.  The supplementary code generates every reported table and figure from fixed seeds.  The UCI data are not redistributed; the code documents the public source and expected file layout.

\section*{AI use statement}
Generative AI tools (OpenAI GPT-5.6 Pro) were used to help formalize parts of the methodology, formulate and check mathematical claims, assist with proof development, design and implement synthetic and UCI-HAR experiments, interpret numerical outputs, and edit and format the manuscript.  All AI-assisted derivations, code, numerical results, citations, and prose were reviewed by the authors, who take responsibility for the final content and for independently verifying it before submission.

\bibliography{references}
\bibliographystyle{iclr2027_conference}
\appendix
\section{Degree-two representation: complete calculations}
\label{app:representation}

This section proves \cref{prop:isometry-main} and records the identities used by the concentration argument.  Throughout, $\mu$ is the product uniform probability measure on $[-1,1]^d$.

\subsection{Orthogonal projection coefficients}
The one-dimensional Legendre polynomials used in the paper satisfy
\[
\int p_0^2\,d\mu_1=1,\qquad
\int p_1^2\,d\mu_1=\frac13,\qquad
\int p_2^2\,d\mu_1=\frac15,
\]
and distinct polynomials among $p_0,p_1,p_2$ are orthogonal.  Because $\mu$ is a product measure, an orthogonal basis for the total-degree-at-most-two subspace is
\[
1,\quad \{x_j\}_{j=1}^d,\quad \{p_2(x_j)\}_{j=1}^d,
\quad \{x_ix_j\}_{1\le i<j\le d}.
\]
If $f=dP/d\mu$, the coefficient of a basis element $\phi$ in the orthogonal projection is $\langle f,\phi\rangle/\|\phi\|_2^2=\E_f\phi(X)/\|\phi\|_2^2$.  This gives
\[
a_j=3\E_fX_j,\qquad b_j=5\E_fp_2(X_j),\qquad
\gamma_{ij}=9\E_f(X_iX_j),
\]
and proves \eqref{eq:p2-main}.

The expectation matrix $M(f)=\E_fH_2(X)$ has entries
\begin{equation}
M(f)_{00}=1,\qquad M(f)_{0j}=\frac{a_j}{\sqrt6},\qquad
M(f)_{jj}=\frac{b_j}{\sqrt5},\qquad
M(f)_{ij}=\frac{\gamma_{ij}}{3\sqrt2}\quad(i<j).
\label{eq:appendix-M-coeff}
\end{equation}
Thus the apparently nonparametric object $P_2f$ is estimated by the ordinary sample mean of a fixed matrix-valued feature map.

\subsection{Proof of the isometry}
\begin{proof}[Proof of the first assertion in \cref{prop:isometry-main}]
By orthogonality,
\begin{equation}
\norm[L^2(\mu)]{P_2f}^2
=1+\sum_{j=1}^d\frac{a_j^2}{3}
 +\sum_{j=1}^d\frac{b_j^2}{5}
 +\sum_{i<j}\frac{\gamma_{ij}^2}{9}.
\label{eq:p2-norm-app}
\end{equation}
Using \eqref{eq:appendix-M-coeff} and remembering that every off-diagonal entry occurs twice in the Frobenius norm,
\begin{align*}
\norm[\F]{M(f)}^2
&=1+2\sum_{j=1}^d\frac{a_j^2}{6}
 +\sum_{j=1}^d\frac{b_j^2}{5}
 +2\sum_{i<j}\frac{\gamma_{ij}^2}{18}\\
&=1+\sum_{j=1}^d\frac{a_j^2}{3}
 +\sum_{j=1}^d\frac{b_j^2}{5}
 +\sum_{i<j}\frac{\gamma_{ij}^2}{9}.
\end{align*}
The same calculation applied to the coefficient differences of $f$ and $g$ proves
\[
\norm[\F]{M(f)-M(g)}^2
=\norm[L^2(\mu)]{P_2(f-g)}^2.
\]
\end{proof}

\subsection{Envelope and isotropic second moment}
\begin{lemma}[Uniform matrix envelope]
\label{lem:H-envelope-app}
For every $x\in[-1,1]^d$,
\[
\norm[\op]{H_2(x)}\le\norm[\F]{H_2(x)}\le3d.
\]
Consequently, for any density $f$, $\|M(f)\|_{\op}\le3d$ and
$\|H_2(X)-M(f)\|_{\op}\le6d$ almost surely.
\end{lemma}
\begin{proof}
Since $|p_2(x)|\le1$ on $[-1,1]$, direct summation gives
\begin{align*}
\norm[\F]{H_2(x)}^2
&\le 1+2d\frac{9}{6}+5d
  +2\binom d2\frac92\\
&=1+8d+\frac92d(d-1)
\le9d^2,
\end{align*}
for every $d\ge1$.  Jensen's inequality gives
$\|M(f)\|_{\op}\le\E_f\|H_2(X)\|_{\op}\le3d$, and the centered bound follows by the triangle inequality.
\end{proof}

\begin{lemma}[Exact reference second moment]
\label{lem:H-isotropy-app}
If $X\sim\mu$, then
\[
\E_\mu[H_2(X)^2]=\left(1+\frac d2\right)I_{d+1}.
\]
\end{lemma}
\begin{proof}
Write indices as $0,1,\ldots,d$.  For the $(0,0)$ entry,
\[
\E(H_2^2)_{00}=1+\sum_{j=1}^d\frac32\E X_j^2
=1+\frac d2.
\]
For $j\ge1$,
\begin{align*}
\E(H_2^2)_{jj}
&=\frac32\E X_j^2+5\E p_2(X_j)^2
  +\sum_{k\ne j}\frac92\E(X_j^2X_k^2)\\
&=\frac12+1+\frac{d-1}{2}=1+\frac d2.
\end{align*}
Every off-diagonal entry has expectation zero.  For example,
$(H_2^2)_{0j}$ is a sum of terms proportional to $X_j$, $X_jp_2(X_j)$, or $X_jX_k^2$, all of which have zero expectation.  For $j\ne k$, every term in $(H_2^2)_{jk}$ contains an odd power of either $X_j$ or $X_k$.  Independence and symmetry of the coordinates complete the proof.
\end{proof}

If $f\le L$ relative to $\mu$, the last lemma implies the Loewner bound
\begin{equation}
\E_f[H_2(X)^2]
=\int H_2(x)^2 f(x)\,d\mu(x)
\preceq L\left(1+\frac d2\right)I_{d+1}.
\label{eq:density-second-moment-app}
\end{equation}
This is the point at which the bounded-density assumption enters the simple independent-observation theory.

\section{Population matrix CUSUM on a one-change interval}
\label{app:population}

Suppose $(s,e]$ contains exactly one change $b$, with matrix means $M_0$ before $b$ and $M_1$ after $b$, and put $D=M_1-M_0$.

\begin{proof}[Proof of \eqref{eq:population-scalar-main}--\eqref{eq:a-main}]
If $t<b$, the left empirical mean is unbiased for $M_0$, while the right mean has expectation
\[
\frac{b-t}{e-t}M_0+\frac{e-b}{e-t}M_1.
\]
Therefore
\[
\E(\widehat M_R-\widehat M_L)=\frac{e-b}{e-t}D,
\]
and multiplication by the CUSUM normalizer gives the first branch of \eqref{eq:a-main}.  If $t>b$, the right mean is unbiased for $M_1$, while
\[
\E\widehat M_L=\frac{b-s}{t-s}M_0+\frac{t-b}{t-s}M_1,
\]
so the difference equals $(b-s)D/(t-s)$.  The case $t=b$ is immediate.  All three coefficients are positive.
\end{proof}

For $t<b$, $a_b^{s,e}(t)^2$ is a positive constant times $(t-s)/(e-t)$ and is strictly increasing in $t$.  For $t>b$, it is a positive constant times $(e-t)/(t-s)$ and is strictly decreasing.  Thus the population score has a unique maximizer at $b$ whenever $D\ne0$.

\begin{lemma}[Exact squared margin]
\label{lem:exact-margin-app}
For every $s<t<e$, the margin in \eqref{eq:exact-margin-main} holds.
\end{lemma}
\begin{proof}
For $t<b$,
\begin{align*}
(a_b^{s,e}(b))^2-(a_b^{s,e}(t))^2
&=\frac{(b-s)(e-b)}{e-s}
 -\frac{(e-b)^2(t-s)}{(e-s)(e-t)}\\
&=\frac{(e-b)(b-t)}{e-t}.
\end{align*}
For $t>b$, the symmetric calculation gives
$((b-s)(t-b))/(t-s)$.  Multiplication by $\|D\|_F^2$ proves the claim.
\end{proof}

Two useful consequences are worth recording.  If $b-s,e-b\ge\alpha(e-s)$ and the candidate set is also restricted to the same interior region, then
\begin{equation}
\norm[\F]{\DeltaCU^{s,e}(b)}^2-
\norm[\F]{\DeltaCU^{s,e}(t)}^2
\ge \alpha|t-b|\kappa^2.
\label{eq:linear-margin-interior-app}
\end{equation}
Also,
\begin{equation}
\norm[\F]{\DeltaCU^{s,e}(b)}
=\sqrt{\frac{(b-s)(e-b)}{e-s}}\,\kappa
\le\frac{\sqrt{e-s}}{2}\kappa.
\label{eq:peak-upper-app}
\end{equation}

\section{Low-rank perturbation and approximate rank}
\label{app:low-rank}

\begin{proof}[Proof of \cref{prop:pert-main}]
Let $B=\widehat A_{(r)}$ and let $A_r=A_{(r)}$.  Since $B$ is a best rank-$r$ approximation to $\widehat A=A+E$ in Frobenius norm,
\[
\|A+E-B\|_F^2\le\|A+E-A_r\|_F^2.
\]
After cancellation,
\begin{equation}
\|A-B\|_F^2
\le\|A-A_r\|_F^2+2\langle E,B-A_r\rangle.
\label{eq:pert-basic-app}
\end{equation}
The matrix $B-A_r$ has rank at most $2r$, hence
\begin{align*}
|\langle E,B-A_r\rangle|
&\le\|E\|_{\op}\|B-A_r\|_*
\le\sqrt{2r}\|E\|_{\op}\|B-A_r\|_F\\
&\le\sqrt{2r}\|E\|_{\op}
\bigl(\|B-A\|_F+\|A-A_r\|_F\bigr).
\end{align*}
Set $x=\|B-A\|_F$, $a=\|A-A_r\|_F$, and
$u=\sqrt{2r}\|E\|_{\op}$.  Equation \eqref{eq:pert-basic-app} gives
$x^2\le a^2+2u(x+a)$, equivalently $(x-u)^2\le(a+u)^2$.
Thus $x\le a+2u$, proving
\[
\|\widehat A_{(r)}-A\|_F
\le\|A-A_{(r)}\|_F+2\sqrt{2r}\|E\|_{\op}.
\]
For $A=a_b^{s,e}(t)D$ with $\rank(D)\le r$, the approximation term vanishes.  Since $2\sqrt2\le2+\sqrt2$, the displayed one-change bound in the main text follows.
\end{proof}

\subsection{Approximately low-rank jumps}
\label{app:approx-rank}
Define the rank-$r$ tail
\[
\rho_r(D)=\|D-D_{(r)}\|_F.
\]
Because $a_b^{s,e}(t)>0$,
$(a_b^{s,e}(t)D)_{(r)}=a_b^{s,e}(t)D_{(r)}$.  Therefore the preceding proof gives
\begin{equation}
\left\|\bigl(\widehatDeltaCU^{s,e}(t)\bigr)_{(r)}
-a_b^{s,e}(t)D\right\|_F
\le a_b^{s,e}(t)\rho_r(D)
  +2\sqrt{2r}\|\ECU^{s,e}(t)\|_{\op}.
\label{eq:approx-rank-bound-app}
\end{equation}
In particular,
\begin{equation}
\left|\widehat W_r^{s,e}(t)-a_b^{s,e}(t)\kappa\right|
\le a_b^{s,e}(t)\rho_r(D)
  +2\sqrt{2r}\|\ECU^{s,e}(t)\|_{\op}.
\label{eq:approx-rank-score-app}
\end{equation}
Thus approximate rank contributes a deterministic, location-dependent bias rather than changing the stochastic analysis.  On a no-change interval, the sharper identity
\begin{equation}
\|E_{(r)}\|_F\le\sqrt r\|E\|_{\op}
\label{eq:null-rank-bound-app}
\end{equation}
holds because the retained $r$ singular values are each bounded by $\|E\|_{\op}$.

\section{Uniform operator-norm control}
\label{app:concentration}

Let $Y_i=H_2(X_i)-\E H_2(X_i)$.  For a fixed triple $(s,t,e)$, write
$n_L=t-s$, $n_R=e-t$, and $N=e-s$.  Then
\begin{equation}
\ECU^{s,e}(t)=\sum_{i=s+1}^e w_i^{s,e}(t)Y_i,
\label{eq:weighted-cusum-app}
\end{equation}
where
\begin{equation}
w_i^{s,e}(t)=
\begin{cases}
-\sqrt{\dfrac{n_R}{Nn_L}},&s<i\le t,\\[1.2ex]
\phantom{-}\sqrt{\dfrac{n_L}{Nn_R}},&t<i\le e.
\end{cases}
\label{eq:weights-app}
\end{equation}
Direct calculation gives
\begin{equation}
\sum_{i=s+1}^e w_i^2=1,
\qquad
w_{\max}:=\max_i|w_i|
=\sqrt{\frac{n_{\max}}{Nn_{\min}}}.
\label{eq:weight-identities-app}
\end{equation}

\begin{proposition}[Pointwise matrix Bernstein bound]
\label{prop:pointwise-concentration-app}
Suppose the observations are independent and every segment density is bounded by $L$.  Let $p=d+1$ and $u>0$.  For a fixed triple $(s,t,e)$,
\begin{equation}
\Pp\left\{
\|\ECU^{s,e}(t)\|_{\op}
>\sqrt{2L(1+d/2)u}+4d w_{\max}u
\right\}
\le2p e^{-u}.
\label{eq:pointwise-concentration-app}
\end{equation}
\end{proposition}
\begin{proof}
For an observation in segment $k$,
\[
\E Y_i^2=\E H_2(X_i)^2-M_k^2
\preceq\E H_2(X_i)^2
\preceq L(1+d/2)I
\]
by \eqref{eq:density-second-moment-app}.  Therefore the matrix variance parameter of the weighted sum is at most
\[
v=\left\|\sum_iw_i^2\E Y_i^2\right\|_{\op}
\le L(1+d/2)\sum_iw_i^2=L(1+d/2).
\]
By \cref{lem:H-envelope-app}, each summand satisfies
$\|w_iY_i\|_{\op}\le6d w_{\max}=:R$.  The self-adjoint matrix Bernstein inequality \citep{tropp2012user} yields
\[
\Pp\left\{\left\|\sum_iw_iY_i\right\|_{\op}
>\sqrt{2vu}+\frac23Ru\right\}\le2p e^{-u}.
\]
Substituting $v$ and $R$ proves \eqref{eq:pointwise-concentration-app}.
\end{proof}

\begin{proof}[Proof of \cref{thm:uniform-main}]
For every candidate in the central half of an interval,
$n_L,n_R\ge N/4$ and $n_{\max}\le3N/4$.  Hence
\[
w_{\max}\le\sqrt{\frac{3}{N}}\le\frac{2}{\sqrt N}.
\]
All scanned intervals have $N\ge m$, so \cref{prop:pointwise-concentration-app} with
$u=\ell_{\cG,\delta}=\log(2(d+1)|\cG|/\delta)$ gives, for a fixed candidate,
\[
\|\ECU^{s,e}(t)\|_{\op}
\le\sqrt{2L(1+d/2)\ell_{\cG,\delta}}
 +\frac{8d\ell_{\cG,\delta}}{\sqrt m}.
\]
A union bound over $\cG$ proves the first statement.  On a one-change interval, combine this event with \cref{prop:pert-main} and the reverse triangle inequality.  This gives the score deviation $\epsr=c_r\sqrt r\lambda_m(\delta)$.  On a no-change interval, \eqref{eq:null-rank-bound-app} gives the sharper score bound $\sqrt r\lambda_m(\delta)$.
\end{proof}

The cardinality used above is modest.  There are $O(n/h)$ intervals of length $h$ and at most $h$ candidate splits per interval, hence $O(n)$ triples per scale.  Summing over $O(\log(n/m))$ dyadic scales gives
\begin{equation}
|\cG|=O\bigl(n\log(n/m)\bigr).
\label{eq:grid-cardinality-app}
\end{equation}

\section{Seeded geometry and the multiple-change induction}
\label{app:multiple-proof}

For a fully discrete construction, take $m$ to be a positive multiple of four and use dyadic lengths $h_j=2^jm$ not exceeding $n$.  At length $h$, include all intervals
\begin{equation}
I_{h,q}=\left(q\frac h2,\ q\frac h2+h\right]
\quad\text{whose endpoints lie in }[0,n].
\label{eq:seeded-grid-app}
\end{equation}
If a user-specified $m$ is not a multiple of four, it can be rounded to an adjacent admissible integer; changing all constants by at most an additive two handles the finitely many boundary lattice points.  We state the exact geometric argument for \eqref{eq:seeded-grid-app}.

\begin{lemma}[Balanced seeded interval]
\label{lem:balanced-seed-app}
Let $(a,c]$ be an active component and let $b\in(a,c)$ satisfy
$\min(b-a,c-b)\ge h$.  Then there is an interval $I=(s,e]$ of the form \eqref{eq:seeded-grid-app}, contained in $(a,c]$, such that
\[
e-s=h,\qquad b-s\ge h/4,\qquad e-b\ge h/4.
\]
\end{lemma}
\begin{proof}
The central halves
\[
\left[qh/2+h/4,\ qh/2+3h/4\right],\qquad q\in\mathbb Z,
\]
tile the real line: the right endpoint of one is the left endpoint of the next.  Choose $q$ whose central half contains $b$.  Then $b$ is at least $h/4$ from each endpoint of $I_{h,q}$ and at most $3h/4$ from each endpoint.  Since $b$ is at least $h$ from $a$ and $c$, the interval is contained in $(a,c]$.
\end{proof}

\begin{lemma}[An undetected change generates an active isolating interval]
\label{lem:active-isolating-app}
Assume the uniform event of \cref{thm:uniform-main}.  Suppose an undetected change $\eta_k$ lies in an active component and is at least $3\Delta/4$ from both component boundaries.  If $m\le\Delta/8$, then there is a seeded interval $I$ contained in that component with
\[
\frac\Delta8\le |I|<\frac\Delta4,
\]
which contains only $\eta_k$, places it in the central half, and satisfies
\[
\widehat S_I\ge \frac{\sqrt\Delta}{8\sqrt2}\kappa-\epsr.
\]
\end{lemma}
\begin{proof}
Let $h$ be the first dyadic length $2^jm$ that is at least $\Delta/8$.  Then $h<\Delta/4$ (with equality allowed only in an immaterial endpoint case).  Since $3\Delta/4\ge h$, \cref{lem:balanced-seed-app} provides a contained interval with left and right lengths at least $h/4$.  Its length is less than the minimal spacing, so it contains no other change.  At the true split,
\[
a_{\eta_k}^{s,e}(\eta_k)
=\sqrt{\frac{(\eta_k-s)(e-\eta_k)}h}
\ge\frac{\sqrt h}{4}
\ge\frac{\sqrt\Delta}{8\sqrt2}.
\]
The one-change score deviation in \cref{thm:uniform-main} completes the proof.
\end{proof}

\begin{proof}[Proof of \cref{thm:multiple-main}]
We prove the following invariant by induction over the number of selected intervals.

\emph{Invariant.}  The active components are disjoint and contain exactly the undetected change points.  Every undetected change is at least $3\Delta/4$ from the two boundaries of its active component.

At initialization, consecutive true change points and the global endpoints are separated by at least $\Delta$, so the invariant holds.  Assume it holds before an iteration.

First, an interval contained in an active component but containing no undetected change has zero population CUSUM.  By \eqref{eq:null-rank-bound-app} and the uniform event, its score is at most $\sqrt r\lambda_m(\delta)<\tau$; it cannot be active.

Second, every undetected change has an active isolating interval of length below $\Delta/4$ by \cref{lem:active-isolating-app} and the upper threshold inequality in \eqref{eq:threshold-main}.  Hence the algorithm cannot stop while an undetected change remains.

Third, let $I^*=(s^*,e^*]$ be a shortest active interval.  It is no longer than one of the isolating intervals, so $|I^*|<\Delta/4$.  The first step shows that it contains an undetected change, and minimal spacing shows that it contains exactly one, say $\eta_j$.  Deleting $I^*$ therefore removes exactly this change.

It remains to verify the invariant for the two new components.  Consider the nearest remaining change on the left, if it exists.  Since $s^*<\eta_j$ and $\eta_j-s^*<|I^*|<\Delta/4$,
\[
s^*-\eta_{j-1}
=(\eta_j-\eta_{j-1})-(\eta_j-s^*)
>\Delta-\frac\Delta4=\frac{3\Delta}{4}.
\]
The right side is identical.  All other component boundaries are unchanged.  Thus the invariant is preserved.

Each iteration removes exactly one change, and the algorithm cannot terminate early.  It therefore performs exactly $K$ iterations.  Every selected interval has length below $\Delta/4$ and contains its matched change, so its recorded maximizer, which also lies in the interval, obeys $|\widehat b_k-\eta_k|\le\Delta/4$.
\end{proof}

The threshold interval is nonempty under the explicit signal condition
\begin{equation}
\frac{\sqrt\Delta}{8\sqrt2}\kappa
>(c_r+1)\sqrt r\lambda_m(\delta).
\label{eq:snr-multiple-app}
\end{equation}
Ignoring constants and the bounded-summand correction, this is
$\Delta\kappa^2\gtrsim rLd\log(nd/\delta)$.

\section{Preliminary localization from the exact margin}
\label{app:preliminary-localization}

\begin{proof}[Proof of \cref{cor:prelim-main}]
Write $W(t)=\|\DeltaCU^{s,e}(t)\|_F$ and
$\widehat W(t)=\widehat W_r^{s,e}(t)$.  On the uniform event,
$\sup_t|\widehat W(t)-W(t)|\le\epsr$.  Since $\widehat b$ maximizes $\widehat W$,
\begin{equation}
W(b)-W(\widehat b)\le2\epsr.
\label{eq:argmax-dev-app}
\end{equation}
The interior exact margin \eqref{eq:linear-margin-interior-app} and the fact that $W(\widehat b)\le W(b)$ imply
\begin{align*}
W(b)-W(\widehat b)
&=\frac{W(b)^2-W(\widehat b)^2}{W(b)+W(\widehat b)}\\
&\ge\frac{\alpha|\widehat b-b|\kappa^2}{2W(b)}.
\end{align*}
Combining this inequality with \eqref{eq:argmax-dev-app} and the peak bound \eqref{eq:peak-upper-app} yields
\[
|\widehat b-b|
\le\frac{4\epsr W(b)}{\alpha\kappa^2}
\le\frac{2\epsr\sqrt N}{\alpha\kappa}
=\frac{2c_r}{\alpha}\frac{\sqrt{Nr}\lambda_m(\delta)}{\kappa}.
\]
\end{proof}

This calculation explains why a global CUSUM maximizer is an effective isolator but not automatically rate-optimal: the uniform error is compared with a score margin rather than with a local likelihood increment.  The refinement below performs the latter comparison after learning a one-dimensional changing direction.

\section{Cross-fitted refinement}
\label{app:refinement-proof}

We make the scalar condition used in the main theorem explicit.  Conditional on any deterministic matrix $V$ with $\|V\|_F=1$, let
\[
\xi_i(V)=\langle H_2(X_i)-M_k,V\rangle_F
\]
on a pure segment $k$.  Assume the Bernstein moment-generating-function bound
\begin{equation}
\log\E\exp\{\lambda\xi_i(V)\}
\le\frac{\lambda^2\nu^2}{2(1-B|\lambda|)},
\qquad |\lambda|<B^{-1}.
\label{eq:bernstein-mgf-app}
\end{equation}
It implies \eqref{eq:scalar-bernstein-main} and its maximal, line-crossing form.  Independent bounded projections satisfy \eqref{eq:bernstein-mgf-app} with conservative parameters depending on the envelope; sharper distributional assumptions can yield smaller $\nu$ and $B$.

\subsection{Pilot direction alignment}
\begin{lemma}[Normalization preserves the jump direction]
\label{lem:direction-alignment-app}
Let $A=aD$ with $a>0$, $\|D\|_F=\kappa$, and suppose
$\|\widehat A_r-A\|_F\le\epsilon$ and $a\kappa\ge3\epsilon$.  Then, for
$V=\widehat A_r/\|\widehat A_r\|_F$,
\[
\langle D,V\rangle_F\ge\frac\kappa2.
\]
\end{lemma}
\begin{proof}
By Cauchy--Schwarz and the triangle inequality,
\begin{align*}
\langle D,V\rangle_F
&=\frac{a\kappa^2+\langle D,\widehat A_r-A\rangle_F}
{\|\widehat A_r\|_F}\\
&\ge\kappa\frac{a\kappa-\epsilon}{a\kappa+\epsilon}
\ge\frac\kappa2.
\end{align*}
\end{proof}

Let $S=\max_t\|\DeltaCU(t)\|_F$ be the pilot population peak and let $\widehat t$ maximize the pilot empirical score.  Uniform score error $\epsilon$ gives
\[
\|\DeltaCU(\widehat t)\|_F
\ge\widehat W(\widehat t)-\epsilon
\ge\widehat W(b)-\epsilon
\ge S-2\epsilon.
\]
If $S\ge8\epsilon$, then the population signal at $\widehat t$ is at least $6\epsilon$, and \cref{lem:direction-alignment-app} applies.  This proves the final assertion of \cref{thm:refine-main} with $\epsilon=\epsr$.

\subsection{Anchor means and least-squares drift}
Condition on the pilot fold and hence on a direction $V$ satisfying
$\delta_V:=\langle D,V\rangle_F\ge\kappa/2$.  The held-out scalar sequence has means
$\mu_0=\langle M_0,V\rangle_F$ and $\mu_1=\langle M_1,V\rangle_F$, with
$\mu_1-\mu_0=\delta_V$.  Let the two pure anchor estimates be
\[
\widehat\mu_L=\mu_0+u_L,
\qquad
\widehat\mu_R=\mu_1+u_R,
\qquad
\widehat\delta=\widehat\mu_R-\widehat\mu_L.
\]
By \eqref{eq:bernstein-mgf-app}, for an absolute $c>0$,
\begin{equation}
\Pp\{|u_L|>\delta_V/8\}
\vee\Pp\{|u_R|>\delta_V/8\}
\le2\exp\left[-cq\min\left\{
\frac{\delta_V^2}{\nu^2},\frac{\delta_V}{B}\right\}\right].
\label{eq:anchor-tail-app}
\end{equation}
On the event
\begin{equation}
\mathcal A=\{|u_L|\vee|u_R|\le\delta_V/8\},
\label{eq:anchor-event-app}
\end{equation}
we have $3\delta_V/4\le\widehat\delta\le5\delta_V/4$.

Let $\eta$ be the held-out-fold change location.  If $t=\eta+m>\eta$, only the $m$ post-change observations between $\eta+1$ and $t$ switch from the right loss to the left loss.  Writing $Z_i=\mu_1+\xi_i$ there,
\begin{align}
Q(\eta+m)-Q(\eta)
&=\sum_{i=\eta+1}^{\eta+m}
\{(Z_i-\widehat\mu_L)^2-(Z_i-\widehat\mu_R)^2\}\nonumber\\
&=m\widehat\delta(\delta_V-u_L-u_R)
 +2\widehat\delta\sum_{i=\eta+1}^{\eta+m}\xi_i.
\label{eq:right-objective-increment-app}
\end{align}
Similarly, for $t=\eta-m<\eta$ and pre-change centered errors,
\begin{equation}
Q(\eta-m)-Q(\eta)
=m\widehat\delta(\delta_V+u_L+u_R)
 -2\widehat\delta\sum_{i=\eta-m+1}^{\eta}\xi_i.
\label{eq:left-objective-increment-app}
\end{equation}
On $\mathcal A$, both deterministic drifts are at least
$9m\delta_V^2/16$.  Thus a candidate at distance $m$ can beat the true split only if a centered partial sum crosses a line of slope at least $3\delta_V/8$.

\subsection{A maximal Bernstein line-crossing bound}
\begin{lemma}[Positive-drift localization bound]
\label{lem:line-crossing-app}
Under \eqref{eq:bernstein-mgf-app}, there are absolute constants $c,C>0$ such that, for every $h\ge1$,
\begin{equation}
\Pp\left\{\exists m\ge h:
\left|\sum_{i=1}^m\xi_i\right|\ge c m\delta_V\right\}
\le2\exp\left[-\frac{h}{C(\nu^2/\delta_V^2+B/\delta_V)}\right].
\label{eq:line-crossing-tail-app}
\end{equation}
\end{lemma}
\begin{proof}
For the lower crossing, apply \eqref{eq:bernstein-mgf-app} to $-\xi_i$.  Choose
$\lambda=c_0\min\{\delta_V/\nu^2,1/B\}$ for a sufficiently small numerical $c_0$.  The exponential process
\[
\exp\left\{-\lambda\sum_{i=1}^m\xi_i
-m\psi(-\lambda)\right\},
\qquad \psi(-\lambda)=\log\E e^{-\lambda\xi_i},
\]
is a nonnegative supermartingale.  The chosen $\lambda$ ensures
\[
\lambda c\delta_V-\psi(-\lambda)
\ge c_1\min\left\{\frac{\delta_V^2}{\nu^2},
\frac{\delta_V}{B}\right\}.
\]
If $-\sum_{i=1}^m\xi_i\ge cm\delta_V$ for some $m\ge h$, the supermartingale is at least the exponential of $h$ times the last display.  Ville's inequality bounds the crossing probability by that reciprocal.  Apply the same argument to $\xi_i$ and use a union bound.  Finally,
\[
\left[\min\left\{\frac{\delta_V^2}{\nu^2},
\frac{\delta_V}{B}\right\}\right]^{-1}
=\max\left\{\frac{\nu^2}{\delta_V^2},
\frac{B}{\delta_V}\right\}
\le\frac{\nu^2}{\delta_V^2}+\frac{B}{\delta_V}.
\]
\end{proof}

\begin{proof}[Proof of \cref{thm:refine-main}]
Conditional on all pilot directions, apply \eqref{eq:anchor-tail-app} to the $2K$ anchor blocks.  The stated lower bound on $q$, with a sufficiently large constant, makes the probability that any anchor event fails at most $\delta/2$ because $\delta_V\ge\kappa/2$.

On the joint anchor event, \eqref{eq:right-objective-increment-app}--\eqref{eq:left-objective-increment-app} and \cref{lem:line-crossing-app} imply, for each $k$,
\[
\Pp\{|\widetilde\eta_k-\eta_k|>h_k\mid\text{pilot},\mathcal A\}
\le\frac{\delta}{2K}
\]
when
\[
h_k=C\left(\frac{\nu^2}{\delta_{V,k}^2}
+\frac{B}{\delta_{V,k}}\right)\log\frac{K}{\delta}.
\]
A union bound and $\delta_{V,k}\ge\kappa_k/2$ yield \eqref{eq:refined-rate-main}, up to a change in the absolute constant.  Passing from a parity-fold index to the original time scale multiplies the fold error by at most two and introduces at most one rounding unit; the factor two is absorbed into $C$ and the additive unit is displayed explicitly.
\end{proof}

\section{Minimax localization lower bound}
\label{app:lower-bound}

\begin{proof}[Proof of \cref{thm:lower-main}]
It is enough to construct a rank-two subfamily.  For $d\ge2$, let
\begin{equation}
f_\theta(x)=1+3\theta x_1x_2,
\qquad |\theta|\le\frac16.
\label{eq:lower-family-app}
\end{equation}
This is a density relative to $\mu$ because it integrates to one and is bounded below by $1/2$.  Its coordinate marginals are uniform.  Moreover,
\[
\E_{f_\theta}(X_1X_2)
=3\theta\E_\mu X_1^2\E_\mu X_2^2=\frac\theta3.
\]
Thus the only changing entries of $M(f_\theta)-M(f_0)$ are
$(1,2)$ and $(2,1)$, each equal to $\theta/\sqrt2$.  The jump has rank two and Frobenius norm $|\theta|$.

For $|u|\le1/2$, $(1+u)\log(1+u)\le u+C u^2$.  Taking
$u=3\theta X_1X_2$ and using $\E_\mu u=0$ gives
\begin{equation}
\operatorname{KL}(f_\theta,f_0)
=\int f_\theta\log f_\theta\,d\mu
\le C\theta^2.
\label{eq:kl-density-app}
\end{equation}
Choose two sequence models with change locations $\eta_0$ and
$\eta_1=\eta_0+h$, using $f_0$ before the change and $f_\theta$ after it.  Their joint laws differ in only $h$ observations, so
\[
\operatorname{KL}(P_0,P_1)\le Ch\theta^2.
\]
Take $h=\lfloor c_1\theta^{-2}\rfloor$ with $c_1$ small enough that this KL divergence is at most $1/8$, and place both locations at least $h$ from the endpoints.  Pinsker's inequality gives
$\operatorname{TV}(P_0,P_1)\le1/4$.  The two-point Le Cam bound then yields, for every estimator,
\begin{align*}
\max_{j\in\{0,1\}}\E_{P_j}|\widehat\eta-\eta_j|
&\ge\frac h2\inf_\phi
\{P_0(\phi=1)+P_1(\phi=0)\}\\
&\ge\frac h2(1-\operatorname{TV}(P_0,P_1))
\ge c h\ge\frac{c'}{\theta^2}.
\end{align*}
Since the projected jump size is $\kappa=|\theta|$, the asserted lower bound follows.
\end{proof}

\section{Uniform geometric \texorpdfstring{$\beta$}{beta}-mixing extension}
\label{app:mixing}

The population calculations use only linearity of expectation and are unchanged under temporal dependence.  Dependence changes only the control of the centered weighted sum.

Let $Z_i=H_2(X_i)-\E H_2(X_i)$ and suppose the triangular array is uniformly absolutely regular:
\begin{equation}
\sup_n\sup_{1\le j\le n-h}
\beta\bigl(\sigma(Z_i:i\le j),\sigma(Z_i:i\ge j+h)\bigr)
\le e^{-c_\beta(h-1)},\qquad h\ge1,
\label{eq:beta-mixing-app}
\end{equation}
with $\|Z_i\|_{\op}\le B_d$ almost surely.  For a fixed triple, put
$A_j=w_{s+j}^{s,e}(t)Z_{s+j}$, $j=1,\ldots,N$.  Deterministic measurable transformations cannot increase absolute-regularity coefficients, so the sequence $A_j$ satisfies \eqref{eq:beta-mixing-app}.  Also
\[
\lambda_{\max}(A_j)\le B_d w_{\max}.
\]
Define the dependent variance proxy
\begin{equation}
v_{s,e,t}^2
=\sup_{\varnothing\ne J\subseteq\{1,\ldots,N\}}
\frac1{|J|}\lambda_{\max}\E\left(\sum_{j\in J}A_j\right)^2.
\label{eq:mix-variance-proxy-app}
\end{equation}
It retains the lagged cross-moments that vanish under independence.

A direct two-sided consequence of Theorem 1 of \citet{banna2016bernstein} is that, for a universal $C>0$,
\begin{equation}
\Pp\{\|\ECU^{s,e}(t)\|_{\op}\ge x\}
\le2(d+1)\exp\left[
-\frac{Cx^2}{Nv_{s,e,t}^2+c_\beta^{-1}B_d^2w_{\max}^2
+xB_dw_{\max}\gamma(c_\beta,N)}
\right],
\label{eq:BMY-cusum-app}
\end{equation}
where
\begin{equation}
\gamma(c,N)=\frac{\log N}{\log2}
\max\left\{2,\frac{32\log N}{c\log2}\right\}.
\label{eq:BMY-gamma-app}
\end{equation}
The same variance proxy appears for $-A_j$, which justifies the factor two.

For a finite scan set $\cG$, define $\lambda_m^{\mathrm{mix}}(\delta)$ as any value $x$ for which the right-hand side of \eqref{eq:BMY-cusum-app} is at most $\delta/|\cG|$ uniformly over $(s,t,e)\in\cG$.  Then
\begin{equation}
\Pp\left\{
\sup_{(s,t,e)\in\cG}\|\ECU^{s,e}(t)\|_{\op}
\le\lambda_m^{\mathrm{mix}}(\delta)
\right\}\ge1-\delta.
\label{eq:uniform-mixing-event-app}
\end{equation}
Every argument in \cref{app:low-rank,app:multiple-proof,app:preliminary-localization} is deterministic conditional on such a uniform event.  Therefore the low-rank score bound, the full multiple-change induction, and the preliminary localization result hold verbatim after replacing $\lambda_m$ by $\lambda_m^{\mathrm{mix}}$.  No global stationarity of the piecewise-distribution sequence is required for this transfer; the uniform mixing condition is imposed directly on the centered array.

\section{Implementation and experimental details}
\label{app:experiments}

\subsection{Feature computation and scan complexity}
The code never constructs a density estimate.  For a block $X\in\R^{q\times d}$, it forms the lower-right part of $\sum_iH_2(X_i)$ through a matrix multiplication plus a diagonal correction.  This costs $O(qd^2)$ arithmetic.  A naive full prefix array costs $O(nd^2)$ memory.  The multiple-change implementation instead stores prefix matrices only at interval endpoints and coarse candidate locations, reducing memory to $O(|\mathcal P|d^2)$ for a selected index set $\mathcal P$.

A dense symmetric eigendecomposition costs $O(d^3)$ per scanned candidate and is used in the released reference code for numerical robustness at $d\le200$.  An implementation needing only the largest $r$ eigenvalues in magnitude can use a Lanczos-type partial eigensolver, reducing the dominant per-candidate spectral cost to approximately $O(rd^2)$ matrix-vector work.  In the large multiple-change experiment, each interval is first scanned on at most 80 coarse candidates, and only selected intervals receive an exact local rescan.  The theory is stated for the full central candidate grid; the coarse-to-fine implementation is a computational approximation.

All reported timings are single-process wall-clock measurements in Python/NumPy on an Intel Xeon Platinum 8573C virtual machine.  The environment used Python 3.13, NumPy 2.3, SciPy 1.17, pandas 2.2, and Matplotlib 3.10.

\subsection{Single-change Gaussian-copula experiment}
For correlation $\rho$, generate a bivariate standard normal pair $(G_1,G_2)$ with correlation $\rho$ and set
\[
X_1=2\Phi(G_1)-1,\qquad X_2=2\Phi(G_2)-1.
\]
All remaining coordinates are independent $\operatorname{Unif}[-1,1]$.  The marginals are uniform for every $\rho$, so all coordinate means and all marginal quadratic coefficients are unchanged.  Only the $(1,2)$ interaction changes.  Its projected jump has rank two and Frobenius norm equal to the Spearman correlation
\[
\kappa=\rho_S=\frac6\pi\arcsin(\rho/2)
\]
when the correlation changes from $0$ to $\rho$.  We use $\rho=0.85$, a change at $n/2$, $n=60d$, and 50 independent replications.

The odd observations are used for the preliminary direction scan and the even observations for local refinement.  Candidate scans use at most 250 evenly spaced coarse points (200 at $d=200$), followed by an exact scan in the neighboring coarse cell.  The held-out scalar least-squares objective uses the first and last 20\% of the fold as anchors.  The oracle baseline knows the affected scalar feature $3X_1X_2$.  The Gaussian-copula density is not uniformly bounded at all corners for nonzero $\rho$; this experiment is therefore a stress test beyond the simple bounded-density sufficient condition used for \cref{thm:uniform-main}.

\begin{table}[ht]
\centering
\caption{Complete single-change summary over 50 replications.  $P_{\le4}$ is the proportion with absolute error at most four and $P_0$ is exact recovery.}
\label{tab:single-app}
\scriptsize
\setlength{\tabcolsep}{3.2pt}
\begin{tabular}{rrlrrrr}
\toprule
$d$ & $n$ & Method & Median & Mean (SD) & $P_{\le4}$ & $P_0$\\
\midrule
20&1200&Full D2&19&44.44 (79.17)&0.24&0.10\\
20&1200&\LRD{} preliminary&6&38.80 (100.13)&0.46&0.20\\
20&1200&\LRD{} refined&11&31.68 (66.57)&0.40&0.18\\
20&1200&Mean CUSUM&419&355.08 (131.02)&0.00&0.00\\
20&1200&Oracle&4&9.68 (12.78)&0.56&0.28\\
\addlinespace
50&3000&Full D2&16&76.76 (167.82)&0.26&0.08\\
50&3000&\LRD{} preliminary&6&16.36 (25.27)&0.46&0.18\\
50&3000&\LRD{} refined&10&14.84 (23.77)&0.32&0.16\\
50&3000&Mean CUSUM&684&660.64 (344.87)&0.00&0.00\\
50&3000&Oracle&4&6.56 (8.66)&0.58&0.30\\
\addlinespace
100&6000&Full D2&41&83.24 (110.01)&0.16&0.04\\
100&6000&\LRD{} preliminary&7&60.36 (330.35)&0.46&0.12\\
100&6000&\LRD{} refined&8&16.68 (29.91)&0.36&0.10\\
100&6000&Mean CUSUM&1416&1352.52 (679.96)&0.00&0.00\\
100&6000&Oracle&6&7.80 (8.76)&0.46&0.20\\
\addlinespace
200&12000&Full D2&96&170.20 (223.53)&0.02&0.02\\
200&12000&\LRD{} preliminary&4&14.84 (29.86)&0.52&0.22\\
200&12000&\LRD{} refined&4&10.96 (16.91)&0.56&0.24\\
200&12000&Mean CUSUM&3166&2753.32 (1609.97)&0.00&0.00\\
200&12000&Oracle&4&6.60 (10.13)&0.58&0.30\\
\bottomrule
\end{tabular}
\end{table}

The large standard deviations for a few preliminary settings reflect rare pilot scans that lock onto a distant noise peak.  The cross-fitted refinement substantially reduces this tail at $d=100$ and $d=200$; at the smallest two dimensions, its independent held-out noise can slightly increase the median even while reducing some large errors.  This is why the paper reports both stages rather than asserting uniform finite-sample improvement.

\subsection{Multiple-change experiment}
We use $d=100$, $n=24000$, change points $(6000,12000,18000)$, and copula-correlation states
$(0.95,-0.95,0.95,-0.95)$.  Each projected jump is rank two.  The training-fold minimum interval length is 700.  The threshold 30 is a fixed reproducibility choice rounded above the maxima of six independent null scans:
\[
29.076,\ 28.607,\ 29.517,\ 28.672,\ 28.850,\ 29.776.
\]
This small calibration is not presented as a formal finite-sample level guarantee; the theorem instead supplies an analytic separation condition, and a larger null bootstrap can be used in applications.

\begin{table}[ht]
\centering
\caption{Multiple-change summary over 12 independent replications.  Errors are Hausdorff distances in original time units.}
\label{tab:multiple-app}
\small
\begin{tabular}{rrrrrrrr}
\toprule
$d$ & $n$ & $K$ & $P(\widehat K=K)$ & Med. pre. & Mean pre. & Med. ref. & Mean ref.\\
\midrule
100&24000&3&1.00&2.0&3.5&1.0&1.67\\
\bottomrule
\end{tabular}
\end{table}
All 12 refined outputs place every estimated change within 10 observations of its matched truth.  Mean end-to-end runtime is 8.93 seconds.

\subsection{UCI Human Activity Recognition experiment}
The public UCI HAR data contain 561 precomputed smartphone features in $[-1,1]$ \citep{anguita2013har}.  We concatenate the official train and test partitions, retain activity labels 1 (WALKING) and 3 (WALKING\_DOWNSTAIRS), and choose the 128 coordinates of largest pooled variance.  In every replication, 400 rows are sampled with replacement from each activity.  Each row is multiplied by an independent Rademacher sign.  The population first moment of the sign-symmetrized distribution is exactly zero, while every degree-two product is unchanged because the sign squares to one.  Hence the benchmark isolates a non-mean distribution shift.

\begin{table}[ht]
\centering
\caption{UCI-HAR localization over 50 randomized replications, $d=128$ and $n=800$.}
\label{tab:har-app}
\small
\begin{tabular}{lrrrr}
\toprule
Method & Median & Mean (SD) & $P_0$ & $P_{\le4}$\\
\midrule
Full D2&0&1.12 (2.29)&0.68&0.94\\
\LRD{} preliminary&0&1.12 (2.29)&0.68&0.94\\
\LRD{} refined&0&1.08 (2.11)&0.72&0.94\\
Mean CUSUM&256&221.56 (91.34)&0.00&0.00\\
\bottomrule
\end{tabular}
\end{table}

\subsection{Reproduction commands}
The supplementary source contains \texttt{code/simulations.py} and
\texttt{code/har\_experiment.py}.  From the source root, regenerate the synthetic experiments with
\begin{quote}
\footnotesize\ttfamily
python code/simulations.py --out results\\
\hspace*{1em}--single-reps 50 --multiple-reps 12 --seed 250813
\end{quote}
For UCI HAR, first download the public archive, concatenate the official train and test feature/label files, and run
\begin{quote}
\footnotesize\ttfamily
python code/har\_experiment.py --x X\_all.txt --y y\_all.txt\\
\hspace*{1em}--out results --reps 50 --dimension 128 --segment-size 400
\end{quote}
The scripts write raw replication-level CSV files, summaries, and all figures used in the manuscript.

\end{document}